%% file: FPFT.tex
\documentclass{article} 
\usepackage{iclr2020_conference,times}

\input{math_commands.tex}

\usepackage{hyperref}
\usepackage{url}

\usepackage{graphicx}
\usepackage{graphics}

\usepackage{amsthm}
\usepackage{amssymb}
\usepackage{amsmath}
\usepackage{epsfig, graphics}
\usepackage{latexsym}
\usepackage{bm}
\usepackage[parfill]{parskip}
\usepackage{subfigure}
\usepackage{xcolor}
\usepackage[ruled, linesnumbered]{algorithm2e}
\usepackage[utf8]{inputenc}
\usepackage[english]{babel}
\usepackage{rotating,booktabs,multirow}
\usepackage{tablefootnote}

\newtheorem{theorem}{Theorem}

\newtheorem{lemma}[theorem]{Lemma}
\theoremstyle{definition}
\newtheorem{definition}{Definition}[section]

\newcommand{\method}{PFT\xspace}

\newcommand*{\QEDB}{\hfill\ensuremath{\Box}}%
\newcommand*{\naive}{na\"ive }

\title{Fast Partial Fourier Transform}

\author{Yong-chan Park, Jun-Gi Jang, U Kang  \\
Seoul National University\\
\texttt{\{wjdakf3948,elnino4,ukang\}@snu.ac.kr}
}
\iclrfinalcopy

\begin{document}

\maketitle

\begin{abstract}
    \input{000abstract}
\end{abstract}

\section{Introduction}
    \label{sec:intro}
    \input{010intro}

\section{Related Work}
    \label{sec:related}
    \input{050related}
\section{Proposed Method}
    \label{sec:proposed}
    \input{030method}

\section{Experiments}
    \label{sec:experiments}
    \input{040experiments}


\section{Conclusions}
    \label{sec:conclusions}
    \input{060conclusions}

\bibliography{bib/other}
\bibliographystyle{iclr2020_conference}

\newpage
\appendix
\input{099appendix}

\end{document}

%% file: math_commands.tex
\usepackage{amsmath,amsfonts,bm}

\def\eqref#1{equation~\ref{#1}}

\def\1{\bm{1}}

\DeclareMathAlphabet{\mathsfit}{\encodingdefault}{\sfdefault}{m}{sl}
\SetMathAlphabet{\mathsfit}{bold}{\encodingdefault}{\sfdefault}{bx}{n}



%% file: 000abstract.tex
Given a time series vector,
how can we efficiently compute a specified part of Fourier coefficients?
Fast Fourier transform (FFT) is a widely used algorithm that computes the discrete
Fourier transform in many machine learning applications. 
Despite its pervasive use, all known FFT algorithms do not provide a fine-tuning
option for the user to specify one’s demand, that is, the output size
(the number of Fourier coefficients to be computed) is algorithmically determined by the input size.
This matters because not every application using FFT requires the whole spectrum
of the frequency domain, resulting in an inefficiency due to extra computation.

In this paper, we propose a fast Partial Fourier Transform (PFT),
a careful modification of the Cooley-Tukey algorithm that enables
one to specify an arbitrary consecutive range where the coefficients should be computed.
We derive the asymptotic time complexity of PFT
with respect to input and output sizes, as well as its numerical accuracy.
Experimental results show that our algorithm outperforms the state-of-the-art FFT algorithms,
with an order of magnitude of speedup for sufficiently small output sizes without sacrificing accuracy.

%% file: 010intro.tex
How can we efficiently compute a specified part of Fourier coefficients for a given time series vector?
Discrete Fourier transform (DFT) is a crucial task in several application areas,
including anomaly detection (\cite{HouZ07,RasheedPAR09,RenXWYHKXYTZ19}), data center monitoring (\cite{MueenNL10}), and image processing (\cite{ShiYY17}).
Notably, in many such applications, it is well known that the DFT results in strong ``energy-compaction'' or
``sparsity'' in the frequency domain.
That is, the Fourier coefficients of data are mostly small or equal to zero,
having a much smaller support compared to the input size.
Moreover, the support can often be specified in practice (e.g., a few low-frequency coefficients around the origin).
These observations arouse a great interest in an efficient algorithm
capable of computing only a specified part of Fourier coefficients.

Fast Fourier transform (FFT) is an algorithm that rapidly computes the DFT of a vector,
which reduces the arithmetic complexity from \naive $O(N^2)$ to $O(N \log N)$, where $N$ is data size.
Nevertheless, all known FFT algorithms do not provide the aforementioned fine-tuning option for the user,
i.e., the output size (the number of Fourier coefficients to be computed) is algorithmically determined by the input size.
Such a lack of flexibility is often followed by just discarding many unused coefficients,
not to mention the inefficiency due to the extra computation.

In this paper, we propose a fast Partial Fourier Transform (\method),
an efficient algorithm for computing a part of Fourier coefficients.
Specifically, we consider the following problem: \textit{
	given a complex-valued vector $\bm{a}$ of size $N$,
	a non-negative integer $M$, and an integer $\mu$,
	estimate the Fourier coefficients of $\bm{a}$ for the interval $[\mu-M,\mu+M]$.}
The resulting algorithm is of remarkably simple structure, composed of
several ``smaller'' FFTs combined with linear pre- and post-processing steps;
consequently, we achieve $O(N+M \log M)$ complexity of \method.
When $M\ll N$, this is a significant improvement compared to the conventional FFT
which never benefits from the information of $M$, resulting in $O(N \log N)$ complexity, consistently.

To the best of our knowledge, \method is the first DFT algorithm that enables one to control the output interval,
providing great versatility and efficiency on the computation.
There have been studies for estimating the top-$k$ (the $k$ largest in magnitude) Fourier coefficients of a given vector
(\cite{SFFT,AAFFT}), yet none of them grants a ``freedom'' of specifying the output interval.
Furthermore, \method does not require the input size to be a power of 2, unlike many other variants of FFT.
This is because the idea of \method derives from a modification of the Cooley-Tukey algorithm \citep{cooley},
which also makes it straightforward to extend the idea to a higher dimensionality
(indeed, we present 2-dimensional \method in Appendix \ref{2d_PFT}).

Through experiments, we show that \method outperforms the state-of-the-art FFT libraries,
FFTW by \cite{FFTW} and Intel Math Kernel Library (MKL),
with an order of magnitude of speedup for sufficiently small output sizes without sacrificing accuracy.

%% file: 050related.tex
We describe related works on Fast Fourier transform (FFT), as well as its applications.

\textbf{Fast Fourier Transform.}
Cooley and Tukey proposed by far the most commonly used FFT algorithm \citep{cooley},
in which a DFT is recursively broken down into several smaller DFTs
provided that the original data is of composite size.
\cite{johnson2006modified} reduced the arithmetic complexity of FFT
to the state-of-the-art $\sim \frac{34}{9}N \log N$, where $N$ is data size.
Meanwhile, a great interest in an algorithm that efficiently computes only a small number of Fourier coefficients
has grown because of the frequently observed energy-compaction property.
A few techniques have been proposed by \cite{SFFT} and \cite{AAFFT},
which estimate the top-$k$ Fourier coefficients of a given vector.
However, all known FFT algorithms lack the ability to efficiently compute only a specified part of Fourier coefficients.

\textbf{Applications of FFT.}
Fast Fourier transform has been widely used for anomaly detection (\cite{HouZ07,RasheedPAR09,RenXWYHKXYTZ19}).
\cite{HouZ07} and~\cite{RenXWYHKXYTZ19} detect anomaly points of a given data by extracting a compact representation with FFT.
\cite{RasheedPAR09} use FFT to detect local spatial outliers which have similar patterns within a region but different patterns from the outside.
Several works (\cite{Pagh13,PhamP13,MalikB18}) exploit FFT for efficient operations.
\cite{Pagh13} leverages FFT to efficiently compute a polynomial kernel used with support vector machines (SVMs).
\cite{MalikB18} propose an efficient tucker decomposition method using FFT.
In addition, FFT has been used for fast training of convolutional neural networks (\cite{MathieuHL13,rippel2015spectral}) and an efficient recommendation model on a heterogeneous graph (\cite{abs-2007-00216}).

%% file: 030method.tex
\subsection{Overview} \label{Sec_ov}
We propose \method, an efficient algorithm for computing a specified part of Fourier coefficients.
The main challenges and our approaches are as follows:

\begin{enumerate}
	\item \textbf{How can we extract essential information for a specified output?}
	Considering the fact that only a specified part of Fourier coefficients should be computed,
	we need to find an algorithm requiring fewer operations than the direct use of conventional FFT.
	This is achievable by carefully modifying the Cooley-Tucky algorithm,
	finding \textit{twiddle factors} (trigonometric constants) with small oscillations,
	and approximating those factors using polynomial functions (Section \ref{Sec_twiddle_small}).

	\item \textbf{How can we decrease approximation costs?}
	The approach given above involves an approximating process, which would be computationally demanding.
	We propose using a \textit{base exponential function},
	by which all data-independent constants can be precomputed,
	so that one can bypass the approximation problem during the run-time (Section \ref{Sec_base_e} and Section \ref{Sec_range}).

	\item \textbf{How can we further reduce numerical computation?}
	We carefully reorder operations and factorize terms in order to alleviate the complexity of \method.
	Such techniques separate all data-independent factors from data-dependent factors,
	allowing further precomputation.
	The arithmetic cost of the resulting algorithm has an asymptotic upper bound $O(N+M \log M)$,
	where $N$ and $M$ are input and output size descriptors, respectively
	(Section \ref{Sec_sum} and Section \ref{Sec_time}).
\end{enumerate}
We describe details of \method from Section \ref{Sec_app} to Section \ref{Sec_sum},
The time complexity and the approximation bound of \method is analyzed in Section \ref{Sec_analy}.

\subsection{Approximation of twiddle factors} \label{Sec_app}

The key of our algorithm is to approximate a part of twiddle factors
with relatively small oscillations by using polynomial functions,
which reduces the computational complexity of DFT due to the mixture of many twiddle factors.
Using polynomial approximation also allows one to carefully control the degree of polynomial (or the number of approximating terms),
enabling fine-tuning the output range and the approximation bound of the estimation.
Our first goal is to find a collection of twiddle factors with small oscillations.
This can be achieved by slightly adjusting the summand of DFT and splitting the summation as in the Cooley-Tukey algorithm
(Section \ref{Sec_twiddle_small}).
Next, using a proper base exponential function, we give an explicit form of polynomial approximation to the twiddle factors
(Section \ref{Sec_base_e}).

\subsubsection{Twiddle factors with small oscillations} \label{Sec_twiddle_small}
Recall that the DFT is defined as follows:
\begin{align} \label{DFT}
\hat{a}_{m}=\sum_{n \in [N]}{a_{n}e^{-2\pi imn/N}},
\end{align}
where $\bm{a}$ is a complex-valued vector of size $N$, and $[\nu]$ denotes $\{0, 1, \cdots, \nu-1\}$ for a positive integer $\nu$. Assume that $N$ is a composite, so there exist $p, q>1$ such that $N=pq$.
The Cooley-Tukey algorithm re-expresses (\ref{DFT}):
\begin{equation} \label{DFTs}
\begin{split}
\hat{a}_{m}=\sum_{k \in [p]}\sum_{l\in [q]} a_{qk+l}e^{-2\pi im(qk+l)/N}
=\sum_{k \in [p]}\sum_{l\in [q]} a_{qk+l} e^{-2\pi iml/N} \cdot e^{-2\pi imk/p},
\end{split}
\end{equation}
yielding two collections of twiddle factors,
namely $\{e^{-2\pi iml/N}\}_{l\in [q]}$ and $\{e^{-2\pi imk/p}\}_{k\in [p]}$.
Consider the problem of computing $\hat{a}_{m}$ for $-M \leq m \leq M$,
where $M\leq N/2$ is a non-negative integer.
In this case, note that the exponent of $e^{-2\pi iml/N}$ ranges from $-2\pi i {M(q-1)}/{N}$ to $+2\pi i {M(q-1)}/{N}$
and that the exponent of $e^{-2\pi imk/p}$ ranges from $-2\pi i {M(p-1)}/{p}$ to $+2\pi i {M(p-1)}/{p}$.
Here $\frac{(q-1)/N}{(p-1)/p} \sim \frac{1}{p}$,
meaning that the first collection contains twiddle factors with smaller oscillations
compared to the second one.
Typically, a function with smaller oscillation results in a better approximation via polynomials.
In this sense, it is reasonable to approximate the first collection of twiddle factors in (\ref{DFTs}) with polynomial functions,
thereby reducing the complexity of the computation due to the mixture of two collections of twiddle factors.
Indeed, one can further reduce the complexity of approximation.
We slightly adjust the summand in (\ref{DFT}) and split it:
\begin{equation} \label{split_a}
\begin{split}
\hat{a}_{m}
&=\sum_{n \in [N]} {a_n e^{-2\pi im(n-q/2)/N} \cdot  e^{-\pi im /p}} \\
&=\sum_{k \in [p]}\sum_{l\in [q]} {a_{qk+l}e^{-2\pi im(l- q/2)/N} \cdot e^{-2\pi imk/p}  \cdot  e^{-\pi im/p} }.
\end{split}
\end{equation}
In (\ref{split_a}), we observe that the range of exponents of the first collection $\{e^{-2\pi im(l-q/2)/N}\}_{l\in [q]}$ of twiddle factors
 is $[-\pi i {M}/{p}, +\pi i {M}/{p}]$,
a contraction by a factor of around 2 when compared with $[-2\pi i {M(q-1)}/{N}, +2\pi i {M(q-1)}/{N}]$,
hence the twiddle factors with even smaller oscillations.
There is an extra twiddle factor $e^{-\pi im/p}$ in (\ref{split_a}).
Note that, however, it depends on neither $k$ nor $l$, so the amount of the additional computation is relatively small.

\subsubsection{Base exponential function} \label{Sec_base_e}
The first collection of twiddle factors in (\ref{split_a}) consists of $q$ distinct exponential functions.
One can apply approximation process for each function in the collection; however, this would be time-consuming.
A more plausible approach is to
1) choose a base exponential function $e^{uix}$ for some fixed $u \in \mathbb{R}$,
2) approximate $e^{uix}$ by using a polynomial,
and 3) exploit a property of exponential functions: the laws of exponents.
Specifically, suppose that we obtained a polynomial $\mathcal{P}(x)$ that approximates $e^{uix}$ on $|x| \leq |\xi|$,
where $u, \xi$ are non-zero real numbers. Consider another exponential function $e^{vix}$ where $v \neq 0$.
Since $e^{vix} = e^{ui(vx/u)}$, the re-scaled polynomial function $\mathcal{P}(vx/u)$ approximates $e^{vix}$ on $|x| \leq |u\xi/v|$.
This observation indicates that once we find an approximation $\mathcal{P}$ to $e^{uix}$ on $|x| \leq |\xi|$ for properly selected $u$ and $\xi$,
all elements belonging to $\{e^{-2\pi im(l-q/2)/N}\}_{l\in [q]}$ can be approximated by re-scaling $\mathcal{P}$.
Fixing a base exponential function also enables precomputing a polynomial that approximates it,
so that one can avoid solving the approximation problem during the run-time.
We further elaborate this idea in a rigorous manner after giving a few definitions (see Definitions \ref{def_P} and \ref{def_xi})
and present a theoretical approximation bound in Theorem \ref{thm_err_bd}.

Let $\| \cdot \|_R$ be the uniform norm (or supremum norm) restricted to a set $R \subseteq \mathbb{R}$, that is,
$\| f \|_R=\sup\{|f(x)|: x\in R\}$ and
$P_{\alpha}$ be the set of polynomials on $\mathbb{R}$ of degree at most $\alpha$.

\begin{definition} \label{def_P}
Given a non-negative integer $\alpha$ and non-zero real numbers $\xi, u$,
we define a polynomial $\mathcal{P}_{\alpha, \xi, u}$ as the best approximation to $e^{uix}$ out of the space $P_{\alpha}$
under the restriction $|x| \leq |\xi|$:
\begin{align*}
\mathcal{P}_{\alpha, \xi, u} := \operatorname*{arg\,min}_{P \in P_{\alpha}} \| P(x)-e^{uix} \|_{|x|\leq |\xi|},
\end{align*}
and $\mathcal{P}_{\alpha, \xi, u}=1$ when $\xi=0$ or $u=0$.
\QEDB
\end{definition}

\cite{smirnov1999best} proved the unique existence of such a polynomial.
Also, a few techniques called \textit{minimax approximation algorithms} for computing $\mathcal{P}_{\alpha, \xi, u}$
are reviewed in \cite{Fraser65}.

\begin{definition} \label{def_xi}
Given a tolerance $\epsilon>0$ and a positive integer $r\geq 1$,
we define $\xi(\epsilon,r)$ to be the scope about the origin such that the exponential function
$e^{\pi ix}$ can be approximated by a polynomial of degree less than $r$ with approximation bound $\epsilon$:
\begin{align*}
\xi(\epsilon,r) := \sup\{\xi \geq 0: \|\mathcal{P}_{r-1, \xi, \pi}(x)-e^{\pi ix} \|_{|x|\leq \xi}  \leq\epsilon \}.
\end{align*}
We express the corresponding polynomial as $\mathcal{P}_{r-1, \xi(\epsilon,r), \pi}(x) = \sum_{j \in [r]} w_{\epsilon,r-1,j}\cdot x^j$. \QEDB
\end{definition}
In Definition \ref{def_xi}, we choose $e^{\pi ix}$ as a base exponential function.
The rationale behind is as follows.
First, using a minimax approximation algorithm,
we precompute $\xi(\epsilon,r)$ and $\{w_{\epsilon,r-1,j}\}_{j \in [r]}$
for several tolerance $\epsilon$'s (e.g. $10^{-1}, 10^{-2}, \cdots$) and positive integer $r$'s (typically $1\leq r \leq 25$).
When $N,M,p$ and $\epsilon$ are given, we find the minimum $r$ satisfying $\xi(\epsilon,r) \geq M/p$.
Then, by the preceding argument, it follows that the re-scaled polynomial function
$\mathcal{P}_{r-1, \xi(\epsilon,r), \pi}(-2x(l-q/2)/N)$ approximates
$e^{-2\pi ix(l-q/2)/N}$ on $|x| \leq |\frac{N}{2(l-q/2)} \cdot \frac{M}{p}|$ for each $l \in [q]$
(note that if $l-q/2=0$, we have $|\frac{N}{2(l-q/2)} \cdot \frac{M}{p}|=\infty$).
Here $ |\frac{N}{2(l-q/2)} \cdot \frac{M}{p}| = |\frac{q}{2l-q} \cdot M| \geq M$ for all $l \in [q]$.
Therefore, we obtain a polynomial approximation on $|m|\leq M$ for each twiddle factor in $\{e^{-2\pi im(l-q/2)/N}\}_{l\in [q]}$,
namely $\{ \mathcal{P}_{r-1, \xi(\epsilon,r), \pi}(-2m(l-q/2)/N) \}_{l \in [q]}$.
Then, it follows from (\ref{split_a}) that
\begin{equation} \label{approx_a}
\begin{split}
\hat{a}_{m}
&\approx \sum_{k \in [p]}\sum_{l\in [q]} {a_{qk+l} \ \mathcal{P}_{r-1, \xi(\epsilon,r), \pi}(-2m(l-q/2)/N) \cdot e^{-2\pi imk/p} \cdot e^{-\pi im/p}}.
\end{split}
\end{equation}
Thus, we obtain an estimation of $\hat{a}_{m}$ for $-M \leq m \leq M$ by approximating the first collection of twiddle factors in (\ref{split_a}).

\subsection{Arbitrarily centered target ranges} \label{Sec_range}
In the previous section, we have focused on the problem of calculating $\hat{a}_m$ for $m$ belonging to $[-M,M]$.
We now consider a more general case:
let us use the term \textbf{target range} to indicate the range where the Fourier coefficients should be calculated,
and $R_{\mu, M}$ to denote $[\mu-M,\mu+M] \cap \mathbb{Z}$, where $\mu \in \mathbb{Z}$.
Note that the previously given method works only when our target range is centered at $\mu=0$.
A slight modification of the algorithm allows the target range to be arbitrarily centered.
One possible approach is as follows: given a complex-valued vector $\bm{x}$ of size $N$, we define $\bm{y}$ as
$y_n = x_n \cdot e^{-2\pi i \mu n/N}$.
Then, the Fourier coefficients of $\bm{x}$ and $\bm{y}$ satisfy the following relationship:
\begin{align*}
\hat{y}_{m} &=\sum_{n \in [N]} x_n \cdot e^{-2\pi i \mu n/N} \cdot e^{-2\pi i m n/N}
 =\sum_{n \in [N]} x_n \cdot e^{-2\pi i (m+\mu) n/N}
 =\hat{x}_{m+\mu}.
\end{align*}
Therefore, the problem of calculating $\hat{x}_{m}$ for $m \in R_{\mu, M}$ is
equivalent to calculating $\hat{y}_{m}$ for $m \in R_{0,M}$, to which our previous method can be applied.
This technique, however, requires extra $N$ multiplications due to the computation of $\bm{y}$.

A better approach, where one can bypass the extra process during the run-time, is to exploit the following lemma
(see Appendix \ref{proof_translation} for the proof).
\begin{lemma} \label{lem_translation}
Given a non-negative integer $\alpha$, non-zero real numbers $\xi, u$, and any real number $\mu$, the following equality holds:
\[
	e^{ui\mu} \cdot \mathcal{P}_{\alpha, \xi, u}(x-\mu)
	= \operatorname*{arg\,min}_{P \in P_{\alpha}} \| P(x)-e^{uix}\|_{|x-\mu|\leq |\xi|}.
\]
\end{lemma}

This observation implies that, in order to obtain a polynomial approximating $e^{uix}$ on $|x-\mu| \leq |\xi|$,
we first find a polynomial $\mathcal{P}$ approximating $e^{uix}$ on $|x| \leq |\xi|$,
then translate $\mathcal{P}$ by $-\mu$ and multiply it with the scalar $e^{ui\mu}$.
Applying this process to the previously obtained approximation polynomials (see Section \ref{Sec_base_e})
yields
$
\{e^{-2\pi i \mu(l- q/2)/N} \cdot \mathcal{P}_{r-1, \xi(\epsilon,r), \pi}(-2(m-\mu)(l-q/2)/N)\}_{l \in [q]}.
$
We substitute these polynomials for the twiddle factors $\{e^{-2\pi im(l-q/2)/N}\}_{l \in [q]}$ in (\ref{split_a}),
which gives the following estimation of $\hat{a}_m$ for $m \in R_{\mu,M}$:
\begin{equation} \label{approx}
\begin{split}
\hat{a}_{m}
&\approx \sum_{k,l} {a_{qk+l} \ e^{-2\pi i \mu(l- q/2)/N} \cdot \mathcal{P}_{r-1, \xi(\epsilon,r), \pi}(-2(m-\mu)(l-q/2)/N) \cdot e^{-2\pi imk/p} \cdot e^{-\pi im/p}} \\
&=\sum_{k,l} {a_{qk+l} \ e^{-2\pi i \mu(l- q/2)/N} \sum_{j} w_{\epsilon,r-1,j} \ (-2(m-\mu)(l-q/2)/N)^j \cdot e^{-2\pi imk/p} \cdot e^{-\pi im/p}} \\
&=\sum_{j}\sum_{k,l} {a_{qk+l} \ e^{-2\pi i \mu(l- q/2)/N} \ w_{\epsilon,r-1,j} \bigg(\frac{m-\mu}{p} \bigg)^j \bigg(1-\frac{2l}{q} \bigg)^j \cdot e^{-2\pi imk/p} \cdot e^{-\pi im/p}},
\end{split}
\end{equation}
where $k \in [p], l \in [q]$, and $j \in [r]$.

\subsection{Efficient Summations} \label{Sec_sum}
We have found that three main summation steps (each being over $j, k$ and $l$) take place when computing the partial Fourier coefficients.
Note that in (\ref{approx}), the innermost summation $\sum_{j}$ is moved to the outermost position,
and the term $-2(m-\mu)(l-q/2)/N$ is factorized into two independent terms, $(m-\mu)/p$ and $1-2l/q$.
Interchanging the order of summations and factorizing the term result in a significant computational benefit;
we elucidate what operator we should utilize for each summation
and how we can save the arithmetic costs from it.
As we will see, the innermost sum over $l$ corresponds to a matrix multiplication,
the second sum over $k$ can be viewed as multiple DFTs,
and the outermost sum over $j$ is an inner product.

For the first sum, let $A=(a_{kl})=a_{qk+l}$ and
$B=(b_{lj})=e^{-2\pi i \mu(l- q/2)/N} \ w_{\epsilon,r-1,j} \ (1-2l/q)^{j}$,
so that (\ref{approx}) can be written as follows:
\begin{equation*} 
\begin{split}
\hat{a}_{m} &\approx
\sum_{j \in [r]}\sum_{k \in [p]}\sum_{l \in [q]} {a_{kl}  b_{lj} \cdot e^{-2\pi imk/p} \cdot ((m-\mu)/p)^j \cdot e^{-\pi im/p}}.
\end{split}
\end{equation*}
Here, note that the matrix $B$ is data-independent (not dependent on $\bm{a}$), and thus can be precomputed.
Indeed, we have already seen that $\{w_{\epsilon,r-1,j}\}_{j \in [r]}$ can be precomputed.
The other factors $e^{-2\pi i \mu(l-q/2)/N}$ and $(1-2l/q)^j$
composing the elements of $B$ can also be precomputed if $(N, M, \mu, p, \epsilon)$ is known in advance.
Thus, as long as the setting $(N, M, \mu, p, \epsilon)$ is unchanged, we can reuse the matrix $B$
for any input data $\bm{a}$ once the configuration phase of \method is completed (Algorithm \ref{algo_config}).
We shall denote the multiplication $A\times B$ as $C=(c_{kj})$:
\begin{align} \label{FFTs}
\hat{a}_{m} &\approx
\sum_{j \in [r]}\sum_{k \in [p]} c_{kj} \cdot e^{-2\pi imk/p}\cdot ((m-\mu)/p)^j \cdot e^{-\pi im/p}.
\end{align}
For each $j \in [r]$, the summation $\hat{c}_{j;m} = \sum_{k \in [p]} c_{kj} \cdot e^{-2\pi imk/p}$ is a DFT of size $p$.
We perform FFT $r$ times for this computation
and denote the corresponding Fourier coefficient as $\hat{c}_{j;m}$, which yields the following estimation of $\hat{a}_m$:
\begin{align} \label{IPs}
\hat{a}_{m}\approx\sum_{j \in [r]}\hat{c}_{j;m} \cdot ((m-\mu)/p)^j \cdot e^{-\pi im/p}.
\end{align}
Note that $\hat{c}_{j;m}$ is a periodic function of period $p$ with respect to $m$,
so we use the coefficient at $m$ modulo $p$ when $m<0$ or $m \geq p$.
Therefore, the $m^{th}$ Fourier coefficient of $\bm{a}$ can be calculated approximately
by the inner product of $\hat{c}_{j;m}$ and $((m-\mu)/p)^j$ with respect to $j$,
followed by a multiplication with the extra twiddle factor $e^{-\pi im/p}$ (we also precompute $((m-\mu)/p)^j$ and $e^{-\pi im/p}$).
The full computation is outlined in Algorithm \ref{algo_pft}.

\IncMargin{1em}
\begin{algorithm} \label{algo_config} 
		\SetKwFunction{FFT}{FFT}
		\SetKwInOut{Input}{input}\SetKwInOut{Output}{output}
		\Input{Input size $N$, output descriptors $M$ and $\mu$, divisor $p$, and tolerance $\epsilon$}
		\Output{Matrix $B$, divisor $p$, and numbers of rows and columns, $q$ and $r$}
		\BlankLine
		$q \leftarrow N/p$ \\
		$r \leftarrow \min \{ r\in\mathbb{N}: \xi(\epsilon, r) \geq {M}/{p} \}$
		\tcp*[f]{Use precomputed $\xi(\epsilon, r)$ } \\
		\For{ $(l, j) \in [q] \times [r]$ }{
		$B[l,j] \leftarrow e^{-2\pi i \mu(l-q/2)/N} \cdot w_{\epsilon,r-1,j} \cdot (1-2l/q)^j$
		\tcp*[f]{Use precomputed $w_{\epsilon,r-1,j}$ } \\
		}
		\caption{Configuration phase of \method}
\end{algorithm}
\DecMargin{1em}

\IncMargin{1em}
\begin{algorithm} \label{algo_pft} 
	\SetKwFunction{FFT}{FFT}
	\SetKwInOut{Input}{input}\SetKwInOut{Output}{output}
	\Input{Vector $\bm{a}$ of size $N$, output descriptors $M$ and $\mu$,
		and configuration results $B, p, q, r$}
	\Output{Vector $\mathcal{E}(\hat{\bm{a}})$ of estimated Fourier coefficients of $\bm{a}$ for $[\mu-M, \mu+M]$}
	\BlankLine
	$A[k,l] \leftarrow a_{qk+l}$ for $k \in [p]$ and $l \in [q]$ \\
	$C \leftarrow A \times B $ \\
	\For{ $j \in [r]$ }{
		$\hat{C}[m, j] \leftarrow \FFT(C[k,j])$ with respect to $k \in [p]$
	}
	\For{ $m \in [\mu-M, \mu+M]$  }{
		$\mathcal{E}(\hat{\bm{a}})[m] \leftarrow \sum_{j=0}^{r-1} \hat{C}[m\%p, j]\cdot ((m-\mu)/p)^j \cdot e^{-\pi im /p}$
	}
	\caption{Computation phase of \method}
\end{algorithm}
\DecMargin{1em}

\subsection{Theoretical analysis} \label{Sec_analy}
We give theoretical analysis regarding the time complexity of \method
as well as its approximation bound.

\subsubsection{Time complexity} \label{Sec_time}
We analyze the time complexity of \method.
Theorem \ref{thm_time_com} shows that the time cost $T(N,M)$ of \method,
where $N$ and $M$ are input and output size descriptors, respectively,
is bounded by $O(N+M \log M)$, provided $N$ has \textit{sufficiently} many divisors.\footnote{
	Note that, in practice, this necessity is not a big concern because
	one can readily control the input size with basic techniques such as zero-padding or re-sampling.
}
When $M\ll N$, this is a significant improvement compared to the conventional FFT
which never benefits from the information of $M$, resulting in $O(N \log N)$ time cost, consistently.
Before presenting the theorem,
we consider an explicit condition which guarantees the ``sufficiently many divisors'' property.
A positive integer is called $b$-\textbf{smooth} if none of its prime factors is greater than $b$.
For example, the 2-smooth integers are equivalent to the powers of 2.
Lemma \ref{lem_smooth} says that if $N$ is a smooth number, then given any $0 < M \leq N$,
one can always find a divisor of $N$ that is tightly bounded by $\Theta(M)$.
We leave the proofs of the lemma and theorem in Appendices \ref{proof_smooth} and \ref{proof_time_com}.

\begin{lemma} \label{lem_smooth}
	Let $b \geq 2$. If $N$ is $b$-smooth and $M \leq N$ is a positive integer,
	then there exists a positive divisor $p$ of $N$ satisfying ${M}/{\sqrt{b}} \leq p < \sqrt{b} M$.
\end{lemma}
\begin{theorem}
\label{thm_time_com}
Fix a tolerance $\epsilon>0$ and an integer $b\geq 2$. If $N$ is $b$-smooth,
then the time complexity $T(N,M)$ of \method has an asymptotic upper bound $O(N+M\log M)$.
\end{theorem}

\subsubsection{Approximation bound}
We now give a theoretical approximation bound of the estimation via the polynomial $\mathcal{P}$.
We denote the estimated Fourier coefficient of $\bm{a}$ as $\mathcal{E}(\hat{\bm{a}})$.
The following theorem states that
the approximation bound is data-dependent of the total weight $\| \bm{a} \|_1$ of the original vector,
where $\|\cdot \|_1$ denotes the $\ell_1$ norm, and the given tolerance $\epsilon$ (see Appendix \ref{proof_err_bd} for the proof).
\begin{theorem} \label{thm_err_bd}
Given a tolerance $\epsilon>0$, the following inequality holds:
\begin{align*}
\|\hat{\bm{a}}-\mathcal{E}(\hat{\bm{a}})\|_{R_{\mu,M}}  \leq  \| \bm{a} \| _1 \cdot \epsilon,
\end{align*}
where $R_{\mu,M}$ is the target range.
\end{theorem}

%% file: 040experiments.tex
Through experiments, the following questions should be answered:

\begin{itemize}
	\item \textbf{Q1. Run-time cost (Section~\ref{subsec:runtime}).}
	How quickly does \method compute a part of Fourier coefficients compared to
	other competitors without sacrificing accuracy?
	\item \textbf{Q2. Effect of hyper-parameter $p$ (Section~\ref{subsec:hyperparam}).}
	How the different choices of divisor $p$ of input size $N$ affect the overall performance of \method?
	\item \textbf{Q3. Anomaly detection (Section~\ref{subsec:anomaly}).}
	How well does \method work for a practical application using FFT (anomaly detection)?
\end{itemize}

\subsection{Experimental setup}
\textbf{Machine.}
All experiments are performed on a machine equipped with Intel Core i7-6700HQ @ 2.60GHz and 8GB of RAM.

\textbf{Datasets.}
We use both synthetic and real-world datasets listed in Table \ref{tab_data}.

\begin{table}[t]
	\centering
		\caption{Detailed information of datasets.}
		\begin{tabular}{@{}llll@{}}
			\toprule
			\textbf{Dataset} & \textbf{Type} & \textbf{Size} & \textbf{Description} \\
			\midrule
			$\{\mathrm{S}_n\}_{n=12}^{22}$ & Synthetic & $2^n$ & Vectors of random real numbers between 0 and 1 \\
			Urban Sound \tablefootnote{\url{https://urbansounddataset.weebly.com/urbansound8k.html}}
			& Real-world & 32000 & Various sound recordings in urban environment \\
			Air Condition \tablefootnote{\url{https://archive.ics.uci.edu/ml/datasets/Appliances+energy+prediction}}
			& Real-world & 19735 & Time series vectors of air condition information \\
			\bottomrule
		\end{tabular}
		\label{tab_data}
\end{table}

\textbf{Competitors.}
We compare \method with two state-of-the-art FFT algorithms, FFTW and MKL.
All of them including \method are implemented in C++.

\begin{enumerate}
	\item \textbf{FFTW}: FFTW\footnote{\url{http://www.fftw.org/index.html}}
	is one of the fastest public implementation for FFT, offering a hardware-specific optimization.
	We use the optimized version of FFTW 3.3.5,
	and do not include the pre-processing for the optimization as the run-time cost.
	\item \textbf{MKL}: Intel Math Kernel Library\footnote{\url{http://software.intel.com/mkl}} (MKL)
	is a library of optimized math routines including FFT, and often shows a better run time result than the FFTW.
	All the experiments are conducted with an Intel processor for the best performance.
	\item \textbf{\method (proposed)}: we use MKL BLAS routines for the matrix multiplication,
	MKL DFTI functions for the batch FFT computation,
	and Intel Integrated Performance Primitives (IPP) library for the post-processing steps
	such as inner product and element-wise multiplication.
\end{enumerate}

\textbf{Measure.}
In all experiments, we use single-precision floating-point format,
and the parameters $p$ and $\epsilon$ are chosen so that the relative $\ell_2$ error is strictly less than $10^{-6}$,
which ensures that the overall estimated coefficients have at least 6 significant figures.
Explicitly,
\[
\textrm{Relative $\ell_2$ Error} =
\sqrt{ \frac{\sum_{m \in \mathcal{R}} |\hat{a}_m - \mathcal{E}(\hat{a})_m |^2}
					{\sum_{m \in \mathcal{R}} |\hat{a}_m |^2 } } < 10^{-6},
\]
where $\hat{\bm{a}}$ is the actual coefficient,
$\mathcal{E}(\hat{\bm{a}})$ is the estimated coefficient,
and $\mathcal{R}$ is the target range.

\subsection{Run-time cost} \label{subsec:runtime}

\textbf{Run time vs. input size.}
We fix the target range to $R_{0, 2^9}$
and evaluate the run time of \method vs. input sizes $N: 2^{12}, 2^{13}, \cdots, 2^{22}$.
The results are averaged over 10 thousands runs.
Figure \ref{E_N} shows how the three competitive algorithms scale with varying input size,
wherein \method outperforms both FFTW and MKL
provided that the output size is small enough $(<10\%)$ compared to the input size.
Consequently, \method achieves up to $21 \times$ speedup compared to its competitors.
Due to the overhead of the $O(N)$ pre- and $O(M)$ post-processing steps, 
we see that \method runs slower than FFT when $M$ is close to $N$, so the
time complexity tends to $O(N+N \log N)$.

\textbf{Run time vs. output size.}
In this experiment, we fix the input size to $N=2^{22}$
and evaluate the run time of \method vs. target ranges $R_{0,2^9}, R_{0,2^{10}}, \cdots, R_{0,2^{18}}$.
The result is illustrated as a run time vs. output size plot (recall that $|R_{0,M}| \simeq 2M$) in Figure \ref{E_M},
where each point is an average over 10 thousands runs.
Note that the run times of FFTW and MKL are consistent
because they do not benefit from the information of the output size $M$.
We find that when the output size is sufficiently smaller than the input size,
\method is up to $13.8 \times$ faster than the other competitors.

\begin{figure}[t!]
	\centering
	\subfigure[]{
		\includegraphics[width=0.48\linewidth]{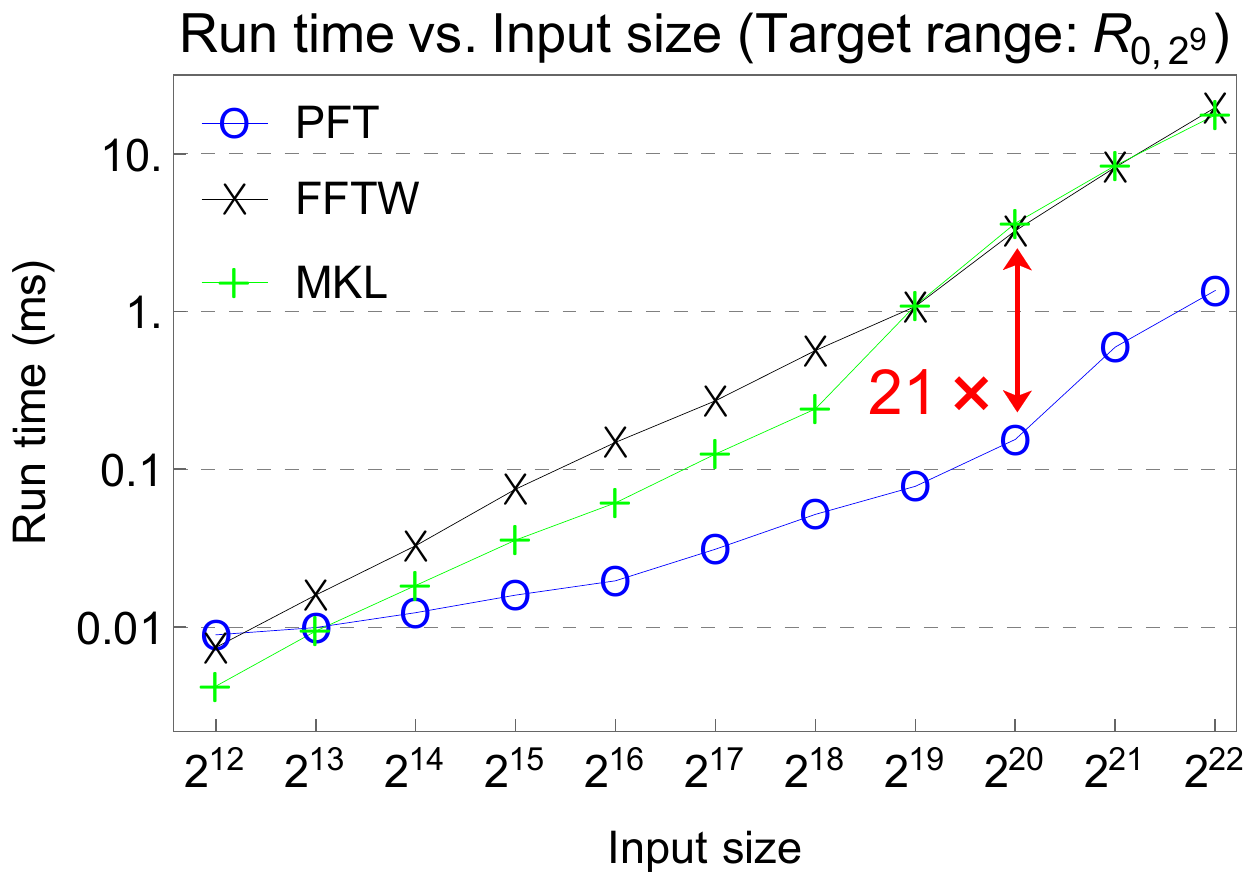}
		\label{E_N}
	}
	\centering
	\subfigure[]{
		\includegraphics[width=0.48\linewidth]{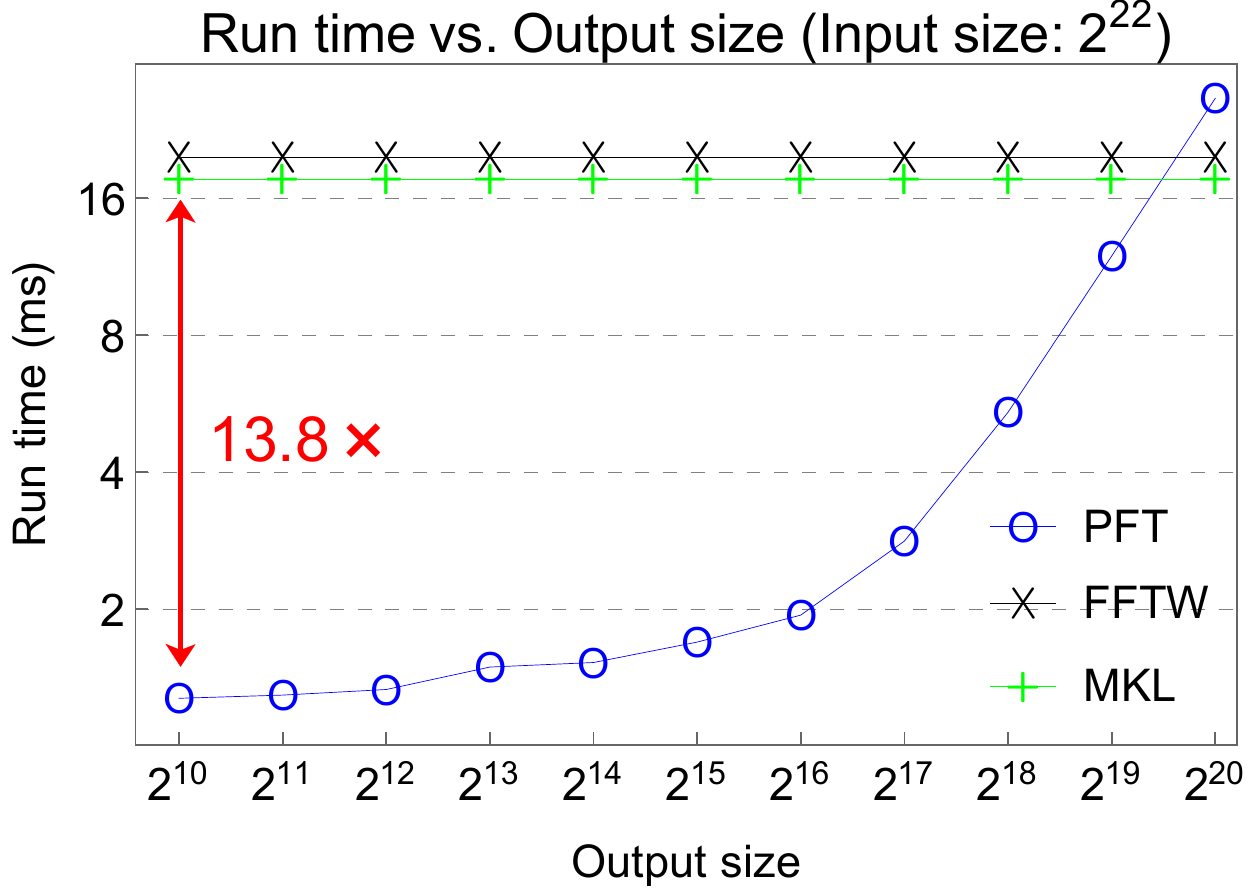}
		\label{E_M}
	}
	\caption{
			\textbf{(a) Run time vs. input size} for target range $R_{0,2^9}$ with $\{\mathrm{S}_n\}_{n=12}^{22}$ datasets, and
			\textbf{(b) run time vs. output size} for $\mathrm{S}_{22}$.
			We observe that \method outperforms both FFTW and MKL
			if the output size is small enough compared to the input size.
			When the output is sufficiently smaller than the input,
			the performance gain is significant: an order of magnitude of speedup.
    }
\end{figure}

\textbf{Real-world data.}
When it comes to real-world data, it is not generally the case that the size of an input vector is a power of 2.
Notably, \method still shows a promising performance
regardless of the fact that the input size is not a power of 2
or not even a \textit{highly composite}\footnotemark number:
a strong indication that our proposed technique is robust for many different applications in real-world.
\footnotetext{By this term, we refer to $b$-smooth integers for sufficiently small $b$ such as $b \leq 7$.}

\begin{figure}[t]
	\centering
	\subfigure[Urban Sound]{
		\includegraphics[width=0.48\linewidth]{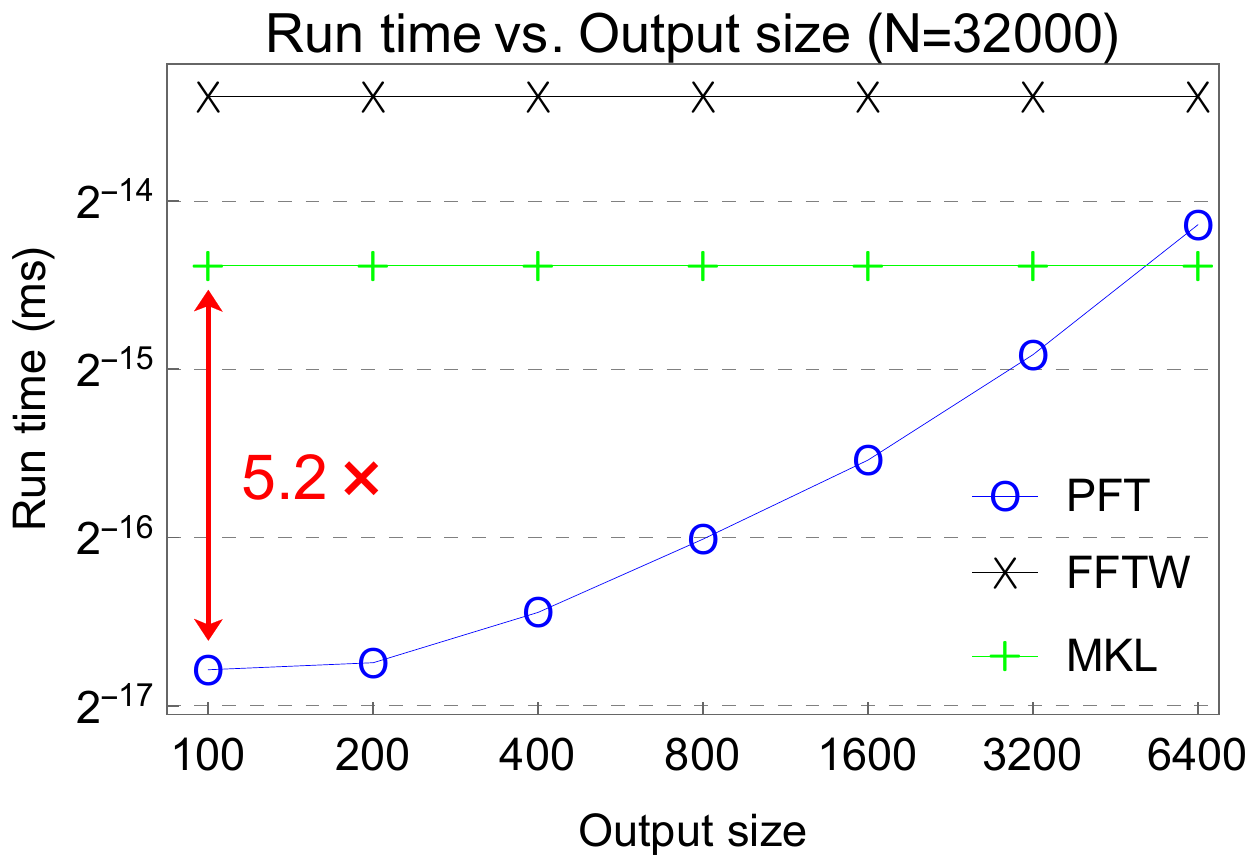}
		\label{exp_urban}
	}
	\centering
	\subfigure[Air Condition]{
		\includegraphics[width=0.48\linewidth]{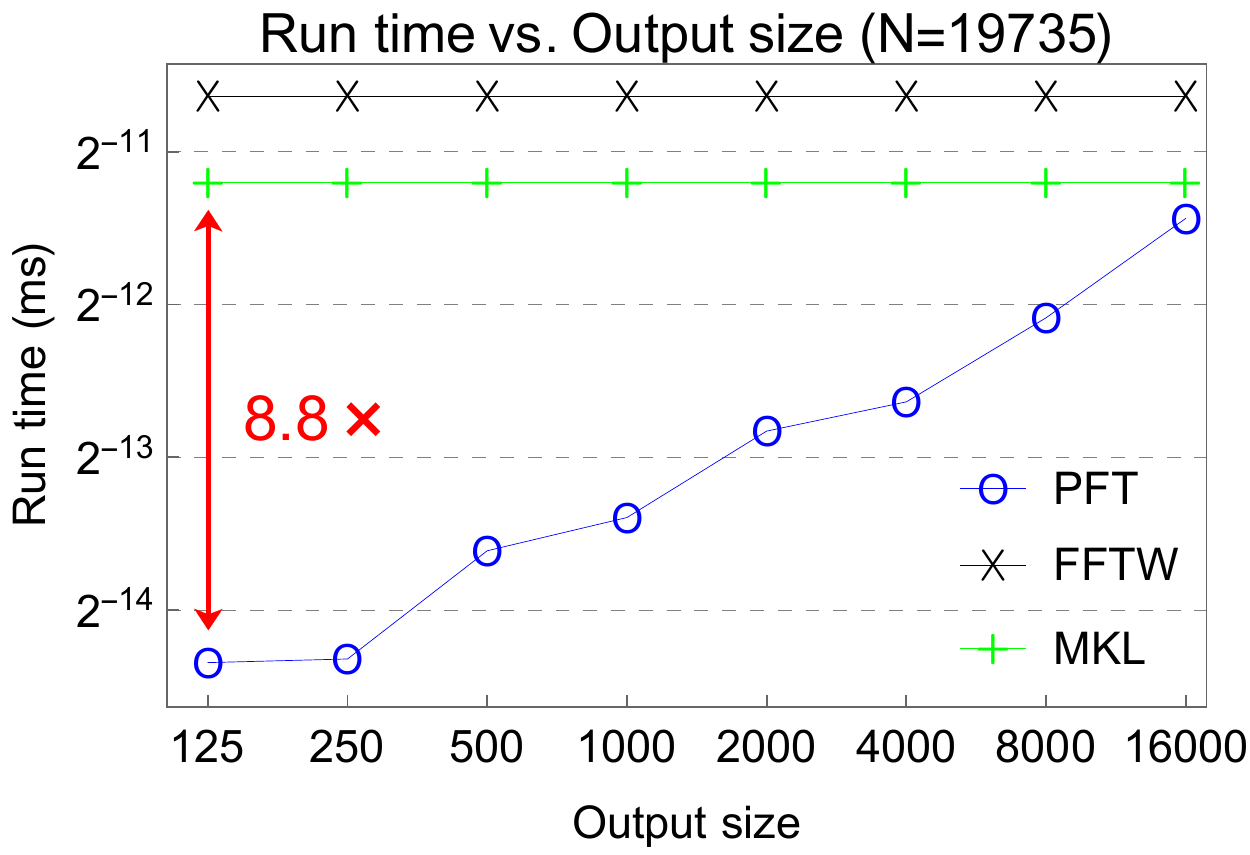}
		\label{exp_air}
	}
	\caption{
			Run time vs. output size results for \textbf{(a) Urban Sound} dataset and \textbf{(b) Air Condition} dataset.
			\method outperforms both FFTW and MKL regardless of the fact that
			the input size is not a power of 2 ($N=2^8 \times 5^3$)
			or not even a highly composite number ($N=5 \times 3947$).
	}
\end{figure}

\begin{itemize}
		\item \textbf{Urban Sound}
		dataset contains various sound recording vectors of size $N=32000=2^8 \times 5^3$.
		We evaluate the run time of \method vs. output sizes: $100$, $200$, $400$, $800$, $1600$, $3200$, and $6400$.
		Figure \ref{exp_urban} illustrates the average run times of the three competitive algorithms.
		We see that \method outperforms both FFTW and MKL if the output size is small enough compared to the input size.

		\item \textbf{Air Condition}
		dataset is composed of time series vectors of size $N=19735=5 \times 3947$.
		Note that $N$ has only two non-trivial divisors, namely $5$ and $3947$,
		forcing one to choose $p=3947$ in any practical settings; 
			if we choose $p=5$, the value $M/p$ often turns out to be \textit{too large}, 
			which results in a poor performance.
		We evaluate the run time of \method vs. output sizes: $125$, $250$, $500$, $1000$, $2000$, $4000$, $8000$, and $16000$,
		as shown in Figure \ref{exp_air}.
		It is noteworthy that \method still outperforms its competitors even in such pathological examples,
		which implies the robustness of our algorithm for various real-world situations.
\end{itemize}

\subsection{Effect of hyper-parameter $p$}
\label{subsec:hyperparam}

To investigate the effect of different choices of $p$, we fix $N=2^{22}$ and
vary the ratio $M/p$ from $1/32$ to $4$ for different target ranges: $R_{0,2^9}, R_{0,2^{10}}, \cdots, R_{0,2^{18}}$.
Table \ref{tab_p} shows the resulting run time for each setting,
where the bold highlights the best choice of $M/p$ for each $M$,
and the missing entries are due to worse performance than the FFT.
One crucial observation is as follows: with the increase of output size, the best choice of the ratio $M/p$ also increases or,
equivalently, the optimal value of $p$ tends to remain stable.
Intuitively, this is the consequence of ``balancing'' the three summation steps (Section \ref{Sec_sum}):
when $M\ll N$, the most computationally expensive operation is the matrix multiplication
with $O(rN)$ time complexity, and thus,
$M/p$ should be small so that the number $r$ of approximating terms decreases,
despite the sacrifices in the batch FFT step requiring $O(rp \log p)$ operations (Appendix \ref{proof_time_com}).
As the $M$ becomes larger, however, more concern is needed regarding the batch FFT and post-processing steps,
so the parameter $p$ should not change rapidly.
This observation, even though we do not provide an explicit formulation,
indicates the possibility that the optimal value of $p$ can be algorithmically auto-selected
given a setting $(N,M,\mu,\epsilon)$, which we leave as a future work.

\begin{table*}[t]
	\centering
	\caption{Average run time (ms) of \method for $N=2^{22}$ with different settings of $M/p$ and $M$.}
	\begin{tabular}{c c c c c c c c c c c}
	\toprule
	\multirow{2}[4]{*}{$M/p$} & \multicolumn{10}{c}{$M$}\\
	\cmidrule(rl){2-11}
	& $2^{9}$ & $2^{10}$ & $2^{11}$ & $2^{12}$ & $2^{13}$ & $2^{14}$ & $2^{15}$ & $2^{16}$ & $2^{17}$ & $2^{18}$ \\
	\cmidrule(r){1-1}\cmidrule(l){2-11}
	\multicolumn{1}{l}{1/32} & \textbf{1.273} & \textbf{1.394} & 1.634 & 2.303 & 5.659 & 14.121 & - & - & - & - \\
	\multicolumn{1}{l}{1/8} & 2.674 & 1.608 & \textbf{1.332} & \textbf{1.491} & 1.860 & 3.020 & 7.711 & - & - & - \\
	\multicolumn{1}{l}{1/2} & 2.627 & 3.738 & 2.717 & 1.678 & \textbf{1.526} & 1.881 & 2.707 & 5.740 & 14.715 & - \\
	\multicolumn{1}{l}{1} & 2.677 & 2.685 & 3.805 & 2.808 & 1.687 & \textbf{1.692} & 2.164 & 3.530 & 7.749 & - \\
	\multicolumn{1}{l}{2} & 4.005 & 2.723 & 2.731 & 3.533 & 2.878 & 1.949 & \textbf{1.940} & \textbf{2.821} & 5.556 & 12.534 \\
	\multicolumn{1}{l}{4} & 4.090 & 4.295 & 2.986 & 2.983 & 4.108 & 3.275 & 2.365 & 2.929 & \textbf{5.411} & \textbf{11.924} \\
	\bottomrule
	\end{tabular}
	\label{tab_p}
\end{table*}

\subsection{Anomaly detection}
\label{subsec:anomaly}

We demonstrate an example of how \method is applied to practical applications.
Here is one simple but fundamental principle:
\textit{replace the ``perform FFT and discard unused coefficients'' procedure with ``just perform \method''.}
Considering the anomaly detection method proposed in \cite{RasheedPAR09},
where one first performs FFT and then inverse FFT with only a few low-frequency coefficients to obtain an estimated fitted curve,
we can directly apply the principle to the method.
To verify this experimentally, we use a time series vector from Air Condition dataset,
and set the target range as $R_{0,125}$ ($\simeq 250$ low-frequency coefficients).
Note that, in this setting, \method results in around $8\times$ speedup compared to the conventional FFT (see Figure \ref{exp_air}).
The top-20 anomalous points detected from the data are presented in Figure \ref{anomaly}.
In particular, we found that replacing FFT with \method does not change the result of top-20 anomaly detection,
with all its computational benefits.

\begin{figure}[t!]
\centering
\includegraphics[width=0.99\linewidth]{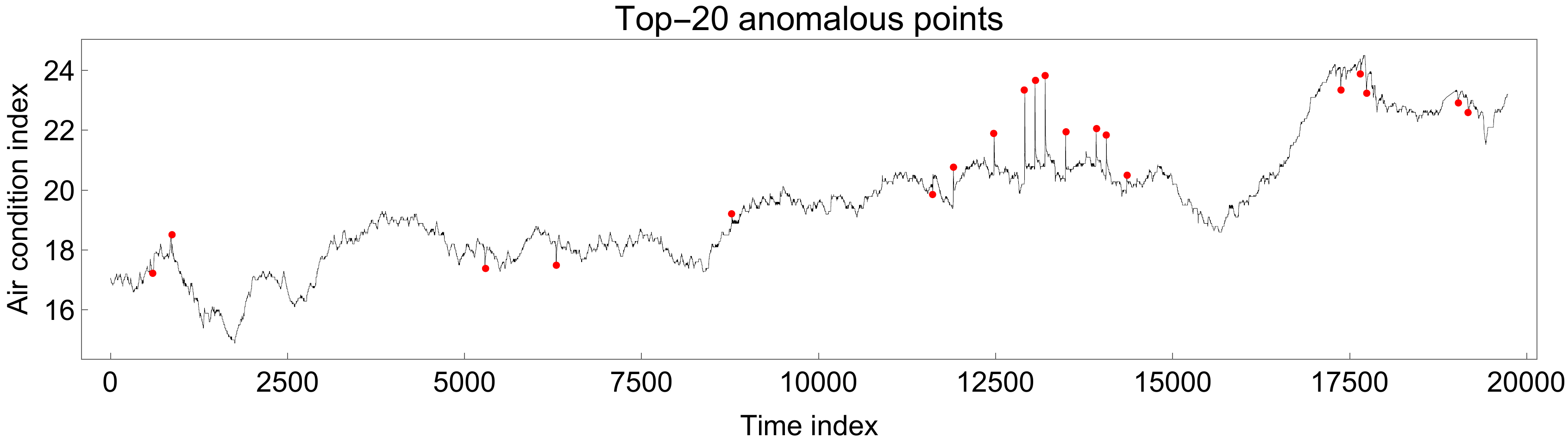}
\caption{
	Top-20 anomalous points detected in Air Condition time-series data, where each red dot denotes a detected anomaly position.
	Note that replacing FFT with \method does not change the result of the detection,
	still reducing the overall time complexity.
}
\label{anomaly}
\end{figure}

%% file: 060conclusions.tex
In this paper, we propose \method (fast Partial Fourier Transform),
an efficient algorithm for computing a specified part of Fourier coefficients.
\method approximates some of twiddle factors with relatively small oscillations using polynomial functions,
reducing the computational complexity of DFT due to the mixture of many twiddle factors.
Experimental results show that our algorithm outperforms the state-of-the-art FFT algorithms, FFTW and MKL,
with an order of magnitude of speedup for sufficiently small output sizes without sacrificing accuracy.
Future works include optimizing the implementation of \method;
for example, the optimal divisor $p$ of input size $N$ might can be algorithmically auto-selected.
We also believe that hardware-specific optimizations (similar to FFTW or MKL) would further increase the performance of \method.

%% file: 099appendix.tex
\section{Proofs}
\subsection{Proof of Lemma \ref{lem_translation}} \label{proof_translation}
\begin{proof}
Recall that the polynomial $\mathcal{P}_{\alpha, \xi, u}$ is defined by 
$\operatorname*{arg\,min}_{P \in P_{\alpha}} \| P(x)-e^{uix}\|_{|x|\leq |\xi|}$.
If $P(x) \in P_{\alpha}$, it is clear that $e^{ui\mu} \cdot P(x-\mu) \in P_{\alpha}$,
because translation and non-zero scalar multiplication on a polynomial do not change its degree.
Thus, we may re-express the definition of $\mathcal{P}_{\alpha, \xi, u}$ as follows:
\begin{align*}
e^{ui\mu} \cdot \mathcal{P}_{\alpha, \xi, u}(x-\mu) 
&= \operatorname*{arg\,min}_{e^{ui\mu} \cdot P(x-\mu) \in P_{\alpha}} \| P(x)-e^{uix}\|_{|x|\leq |\xi|} \\
&= \operatorname*{arg\,min}_{P \in P_{\alpha}} \| e^{-ui\mu} P(x+\mu)-e^{uix}\|_{|x|\leq |\xi|} \\
&= \operatorname*{arg\,min}_{P \in P_{\alpha}} \| P(x+\mu)-e^{ui(x+\mu)}\|_{|x|\leq |\xi|} \\
&= \operatorname*{arg\,min}_{P \in P_{\alpha}} \| P(x)-e^{uix}\|_{|x-\mu|\leq |\xi|},
\end{align*}
where the third equality holds since $| e^{ui\mu}|=1$, and hence the proof.
\end{proof}

\subsection{Proof of Lemma \ref{lem_smooth}} \label{proof_smooth}
\begin{proof}
	Suppose that none of $N$'s divisors belongs to $[M/\sqrt{b}, \sqrt{b} M)$.
	Let $1=p_1<p_2<\cdots<p_d=N$ be the enumeration of all positive divisors of $N$ in increasing order.
	It is clear that $p_1<\sqrt{b} M$ and $M/\sqrt{b}< p_d$ since $b\geq 2$ and $1\leq M\leq N$.
	Then, there exists an $i \in \{1,2,\cdots,d-1 \}$ so that $p_i<M/\sqrt{b}$ and $p_{i+1} \geq \sqrt{b} M$.
	Since $N$ is $b$-smooth and $p_i<N$, at least one of $2p_i, 3p_i, \cdots, bp_i$ must be a divisor of $N$.
	However, this is a contradiction because we have
	$p_{i+1}/p_{i} > (\sqrt{b}M)(M/\sqrt{b})^{-1} = b$,
	so none of $2p_i, 3p_i, \cdots, bp_i$ can be a divisor of $N$, which completes the proof.
\end{proof}

\subsection{Proof of Theorem \ref{thm_time_com}} \label{proof_time_com}
\begin{proof}
	Following the convention in counting FFT operations,
	we assume that all data-independent elements such as configuration results $B, p, q, r$ and twiddle factors are precomputed,
	and thus not included in the run-time cost.
	We begin with construction of the matrix $A$.
	For this, we merely interpret $\bm{a}$ as an array representation for $A$ of size $p\times q=N$.
	Also, recall that the matrix $B$ can be precomputed as described in Section \ref{Sec_sum}.
	For the two matrices $A$ of size $p \times q$ and $B$ of size $q \times r$,
	standard matrix multiplication algorithm has running time of $O(pqr) = O(r \cdot N)$.
	Next, the expression (\ref{FFTs}) contains $r$ DFTs of size $p$. We use FFT multiple times for the computation,
	then it is easy to see that the time cost is given by $O(r\cdot p\log p)$.
	Finally, there are $2M+1$ coefficients to be calculated in (\ref{IPs}),
	each requiring $O(r)$ operations, giving an upper bound $O(r\cdot M)$ for the running time.
	Combining the three upper bounds, we formally express the time complexity $T(N,M)$,
	\begin{align} \label{TNM1}
	T(N,M) = O(r \cdot (N+p\log p +M)).
	\end{align}
	Note that $r$ is only dependent of $\epsilon$ and $M/p$ by its definition. Therefore,
	when $\epsilon$ is fixed, $T(N,M)$ is dependent of the choice of $p$.
	By Lemma \ref{lem_smooth}, we can always find a $p=\Theta(M)$.
	In this case, $M/p$ is bounded, and thus, so is $r$.
	Then, from (\ref{TNM1}), we obtain the following asymptotic upper bound with respect to $N$ and $M$:
	\begin{align*} 
	T(N,M) = O(N+M\log M),
	\end{align*}
	hence the proof.
\end{proof}

\subsection{Proof of Theorem \ref{thm_err_bd}} \label{proof_err_bd}
\begin{proof}
	Let $v=-2 (l- q/2 )/N$.
	By the estimation in (\ref{approx}), the following holds:
	\begin{align*}
	\|\hat{\bm{a}}-\mathcal{E}(\hat{\bm{a}})\|_{R_{\mu,M}}
	&= \| \sum_{k, l}{a_{kl} \big( e^{\pi ivm}- e^{\pi iv\mu} \cdot \mathcal{P}_{r-1, \xi(\epsilon,r), \pi}(v(m-\mu)) \big)e^{-2\pi imk/p} e^{-\pi im/p}}\|_{R_{\mu,M}} \\
	&\leq \sum_{k, l} \|{a_{kl} \big( e^{\pi ivm}-e^{\pi iv\mu} \cdot \mathcal{P}_{r-1, \xi(\epsilon,r), \pi}(v(m-\mu)) \big)e^{-2\pi imk/p} e^{-\pi im/p}}\|_{R_{\mu,M}} \\
	&= \sum_{k, l} |a_{kl}| \cdot \| e^{\pi iv(m-\mu)}-\mathcal{P}_{r-1, \xi(\epsilon,r), \pi}(v(m-\mu)) \|_{R_{\mu,M}},
	\end{align*}
	since $e^{-2\pi imk/p}$ and $e^{-\pi im/p}$ are unit normed functions, and $|e^{\pi iv\mu}|=1$.
	If $l$ ranges from $0$ to $q-1$, then $|v| \leq 2 (q/2) /N=1/p$, and thus, $M|v| \leq M/p \leq \xi(\epsilon,r)$.
	We extend the domain of the RHS from $m \in [\mu-M,\mu+M] \cap \mathbb{Z}$ to $x \in [\mu-M,\mu+M]$
	(note that extending domain never decreases the uniform norm),
	and substitute $v(x-\mu)$ with $x'$, from which it follows that
	\begin{align*}
	\|\hat{\bm{a}}-\mathcal{E}(\hat{\bm{a}})\|_{R_{\mu,M}}
	&\leq \sum_{k, l} |a_{kl}| \cdot \| e^{\pi iv(x-\mu)}-\mathcal{P}_{r-1, \xi(\epsilon,r), \pi}(v(x-\mu)) \|_{|x-\mu| \leq M}  \\
	&= \sum_{k, l} |a_{kl}| \cdot \|e^{\pi ix'}-\mathcal{P}_{r-1, \xi(\epsilon,r), \pi}(x')  \|_{|x'| \leq M|v|} \\
	&\leq \sum_{k, l} |a_{kl}| \cdot \|e^{\pi ix'}-\mathcal{P}_{r-1, \xi(\epsilon,r), \pi}(x')  \|_{|x'| \leq \xi(\epsilon,r)} \\
	&\leq \sum_{k, l} |a_{kl}| \cdot \epsilon \\
	&=\| \bm{a} \| _{\rm1} \cdot \epsilon,
	\end{align*}
	where the second inequality holds since $M|v| \leq \xi(\epsilon,r)$, hence the desired result.
\end{proof}

\section{Two-Dimensional PFT} \label{2d_PFT}
Arguing similarly as the 1-d (dimensional) \method, 
we present an algorithm to compute a part of coefficients of 2-d  DFT
which is defined as follows:
\begin{align} \label{2DFT} 
\hat{a}_{m_1,m_2}=\sum_{(n_1, n_2) \in [N_1]\times[N_2]} a_{n_1,n_2} e^{-2\pi im_1n_1/N_1} e^{-2\pi im_2n_2/N_2},
\end{align}
where $\bm{a}$ is a 2-d complex-valued array of size $N_1 \times N_2$. Let $M_1\leq N_1/2$ and $M_2\leq N_2/2$ be non-negative integers and $(\mu_1, \mu_2) \in \mathbb{Z}^2$.
Our goal is to compute the Fourier coefficients $\hat{a}_{m_1, m_2}$ for $(m_1, m_2)$ belonging to the rectangle,
\[
R_{(\mu_1, \mu_2), (M_1, M_2)} := [\mu_1-M_1, \mu_1+M_1] \times [\mu_2-M_2, \mu_2+M_2] \ \cap \ \mathbb{Z}^2,
\]
for which we use the same terminology ``target range''.
Let $N_1=p_1q_1$ and $N_2=p_2q_2$ be composite integers, where $p_1, p_2, q_1, q_2 > 1$.
The same argument presented in Section \ref{Sec_twiddle_small} gives
\begin{align*}
\hat{a}_{m_1,m_2} =\sum_{k_1, k_2, l_1, l_2}
&a_{q_1k_1+l_1, q_2k_2+l_2} \prod_d e^{-2\pi im_d(l_d- q_d/2 )/N_d} \ e^{-2\pi im_d k_d/p_d}\ e^{-\pi im_d  /p_d} ,
\end{align*}
where $k_1 \in [p_1], k_2 \in [p_2], l_1 \in [q_1], l_2 \in [q_2]$, and $d=1,2$. 
We find the minimum $r_d$ satisfying $\xi(\epsilon,r_d) \geq M_d/p_d$ for each $d$.
Estimating $e^{-2\pi im_d(l_d- q_d/2 )/N_d}$ by $\mathcal{P}_{r_d-1, \xi(\epsilon, r_d), \pi}$ yields
\begin{equation} \label{approx_2}
\begin{split}
\hat{a}_{m_1, m_2} \approx \sum_{j_1, j_2, k_1, k_2, l_1, l_2} a^{(k_1,k_2)}_{l_1 l_2} b^{(1)}_{l_1 j_1}  b^{(2)}_{l_2 j_2} 
	\prod_d e^{-2\pi im_d k_d/p_d} ((m_d-\mu_d)/p_d)^{j_d} \ e^{-\pi im_d /p_d} ,
\end{split}
\end{equation}
where $j_1 \in [r_1], j_2 \in [r_2]$, and
\begin{align*}
&A^{(k_1,k_2)} = (a^{(k_1,k_2)}_{l_1 l_2}) = a_{q_1k_1+l_1, q_2k_2+l_2}, \\
&B^{(d)} = (b^{(d)}_{l_d j_d}) = e^{-2\pi i \mu_d(l_d- q_d/2)/N_d} \ w_{\epsilon,r_d-1, j_d} \ (1-2l_d/q_d)^{j_d}, \quad d=1, 2.
\end{align*}
In (\ref{approx_2}), the summation $\sum_{l_1,l_2} a^{(k_1,k_2)}_{l_1 l_2} b^{(1)}_{l_1 j_1}  b^{(2)}_{l_2 j_2}$
can be written as matrix multiplications,
\begin{align} \label{mm2}
B^{(1) T} \times A^{(k_1, k_2)} \times B^{(2)}.
\end{align}
We denote the result matrix as $C^{(k_1, k_2)} = (c^{(k_1,k_2)}_{j_1  j_2})$.
Next, note that for each $(j_1, j_2) \in [r_1]\times[r_2]$, the operation
$\sum_{k_1,k_2} c^{(k_1,k_2)}_{j_1  j_2} \prod_d e^{-2\pi im_d k_d/p_d}$ is a 2-d DFT of size $p_1 \times p_2$.
Let $\hat{c}^{(j_1,j_2)}_{m_1,m_2}$ be the Fourier coefficients of $c^{(k_1,k_2)}_{j_1  j_2}$ with respect to $(k_1, k_2)$.
Then, we obtain the following estimation of $\hat{a}_{m_1,m_2}$ for $(m_1,m_2) \in R_{(\mu_1, \mu_2), (M_1, M_2)}$:
\begin{equation} \label{ip_2a}
\begin{split}
\hat{a}_{m_1, m_2} \approx \sum_{j_1, j_2}
\hat{c}^{(j_1,j_2)}_{m_1,m_2} \prod_d  ((m_d-\mu_d)/p_d)^{j_d} \ e^{-\pi im_d /p_d}.
\end{split}
\end{equation}
The full computation is outlined in Algorithm \ref{algo_2config} and Algorithm \ref{algo_2pft}.

\IncMargin{1em}
\begin{algorithm}[h] \label{algo_2config} %
	\SetKwFunction{FFT}{FFT}
	\SetKwInOut{Input}{input}\SetKwInOut{Output}{output}
	\Input{Input size $(N_1,N_2)$, output descriptors $(M_1,M_2)$ and $(\mu_1,\mu_2)$, divisors $(p_1,p_2)$, 
		and tolerance $\epsilon$}
	\Output{Configuration results $B^{(1)}, B^{(2)}, p_1,p_2, q_1,q_2, r_1,r_2$}
	\BlankLine
	\For{ $d=1,2$ }{
		$q_d \leftarrow N_d/p_d$ \\
		$r_d \leftarrow \min \{ r\in\mathbb{N}: \xi(\epsilon,r) \geq {M_d}/{p_d} \}$ \\
		\For{ $l \in [q_d], j \in [r_d]$}{
			$B^{(d)}[l,j] \leftarrow e^{-2\pi i \mu_d(l- q_d/2)/N_d} \cdot w_{\epsilon,r_d-1, j} \cdot (1-2l/q_d)^{j}$
		}	
	}
	\caption{Configuration phase of 2-dimensional \method}
\end{algorithm}
\DecMargin{1em}

\IncMargin{1em}
\begin{algorithm}[h] \label{algo_2pft} %
	\SetKwFunction{FFT}{FFT}
	\SetKwInOut{Input}{input}\SetKwInOut{Output}{output}
	\Input{2-d array $\bm{a}$ of size $N_1 \times N_2$, output descriptors $(M_1,M_2)$ and $(\mu_1,\mu_2)$,
		and configuration results $B^{(1)}, B^{(2)}, p_1,p_2, q_1,q_2, r_1,r_2$}
	\Output{2-d array $\mathcal{E}(\hat{\bm{a}})$ of estimated Fourier coefficients of $\bm{a}$ for $R_{(\mu_1, \mu_2), (M_1, M_2)}$}
	\BlankLine
	$A^{(k_1,k_2)}[l_1,l_2] \leftarrow a_{q_1k_1+l_1, q_2k_2+l_2}$ for $k_1 \in [p_1], k_2 \in [p_2], l_1 \in [q_1], l_2 \in [q_2]$ \\
	\For{ $(k_1,k_2) \in [p_1] \times [p_2]$ }{
		$C^{(k_1,k_2)} \leftarrow B^{(1)T} \times A^{(k_1,k_2)} \times B^{(2)} $
	} 
	\For{ $(j_1,j_2) \in [r_1] \times [r_2]$ }{
		$\hat{C}^{(j_1,j_2)}[m_1, m_2] \leftarrow \FFT(C^{(k_1,k_2)}[j_1,j_2])$ with respect to $(k_1, k_2) \in [p_1]\times[p_2]$
	}
	\For{ $(m_1, m_2) \in [\mu_1-M_1,\mu_1+M_1] \times [\mu_2-M_2, \mu_2+M_2]$ }{
		$\mathcal{E}(\hat{\bm{a}})[m_1, m_2] \leftarrow \sum_{j_1 \in [r_1], j_2 \in [r_2]} \hat{C}^{(j_1, j_2)}[m_1\%p_1, m_2\%p_2] 
	 \prod_{d=1,2} ((m_d-\mu_d)/p_d)^{j_d} \ e^{-\pi im_d  /p_d}$ \\
	}
	\caption{Computation phase of 2-dimensional \method}
\end{algorithm}
\DecMargin{1em}

The analysis of 2-d \method is also analogous to the 1-d case.
As in Section \ref{Sec_time}, for a given setting $(N_1, N_2, M_1, M_2, \mu_1, \mu_2, p_1, p_2, \epsilon)$,
we assume that all data-independent constants such as
$B^{(1)}, B^{(2)}$, and any twiddle factors are precomputed.
We shall use the following notations:
\[
N=N_1N_2, \quad M=M_1M_2, \quad p=p_1p_2, \quad q=q_1q_2, \quad r=r_1r_2.
\]
The estimation (\ref{approx_2}) involves matrix multiplications (\ref{mm2}) for each $(k_1, k_2) \in [p_1] \times [p_2]$.
Note that (\ref{mm2}) has two parenthesizations,
namely $(B^{(1)T} A^{(k_1, k_2)}) B^{(2)}$ and $B^{(1)T} (A^{(k_1, k_2)} B^{(2)})$,
each requiring $O(q_2 r_1 (q_1+r_2))$ and $O(q_1 r_2 (q_2+r_1))$ operations, respectively,
which allows one to choose the parenthesization with lower cost.
Without loss of generality, we may assume that the former requires fewer operations.
Then, the total cost of computing $C^{(k_1,k_2)}$ for all $(k_1, k_2)$ is given by $O(p_1 p_2 q_2 r_1(q_1+r_2))=O(r\cdot N)$
since
\[
p_1 p_2 q_2 r_1(q_1+r_2) = \frac{r_1r_2N_1N_2}{r_2}+ \frac{r_1r_2N_1N_2}{q_1} = (\frac{1}{r_2}+\frac{1}{q_1})rN < 2rN.
\]
We next perform $r_1 r_2$ 2-d FFTs of size $p_1 \times p_2$ to calculate $\hat{\bm{c}}^{(j_1, j_2)}$, which takes $O(r \cdot p \log p)$ time.
The remaining computation (\ref{ip_2a}) requires $O(r)$ operations for each $m$, giving an $O(r \cdot M)$ running time.
The time cost $T(N, M)$ of 2-d \method, therefore, can be written as
\begin{align} \label{TNM21} 
T(N,M) = O(r \cdot (N+p\log p +M)).
\end{align}
This is exactly the same form as in the 1-dimensional analysis,
which leads to the following analogy to Theorem \ref{thm_time_com} presented in Section \ref{Sec_time}.

\begin{theorem}
\label{th_smooth2}
Fix a tolerance $\epsilon>0$ and two integers $b_1, b_2\geq 2$.
If $N_1$ is $b_1$-smooth and $N_2$ is $b_2$-smooth, 
then the time complexity $T(N,M)$ of two-dimensional \method has an asymptotic upper bound $O(N+M\log M)$.
\end{theorem}
\begin{proof}
By Lemma \ref{lem_smooth}, we can always find $p_1|N_1$ and $p_2|N_2$ such that $p_1 = \Theta(M_1)$ and $p_2 = \Theta(M_2)$,
giving a tight bound for $r=r_1r_2$. Since  $p=p_1p_2=\Theta(M_1M_2)=\Theta(M)$, we obtain the desired upper bound from (\ref{TNM21}).
\end{proof}

Finally, the following theorem gives an approximation bound of 2-d \method.
\begin{theorem} \label{thm_err_bd2}
	Given $\epsilon>0$, the estimated Fourier coefficient $\mathcal{E}(\hat{\bm{a}})$ in (\ref{approx_2}) satisfies
	\begin{align*}
	\|\hat{\bm{a}}-\mathcal{E}(\hat{\bm{a}})\|_{R_{(\mu_1,\mu_2),(M_1,M_2)}} \leq \|\bm{a}\|_1 \cdot (\epsilon^2+2\epsilon),
	\end{align*}
	where $\| \cdot \|_R$ denotes the uniform norm restricted to $R \subseteq \mathbb{R}^2$.
\end{theorem}
\begin{proof}
	Let $v_d=-2 (l_d- q_d/2 )/N_d$ for $d=1,2$, and $\mathcal{R}=R_{(\mu_1,\mu_2),(M_1,M_2)}$.
	Then, it follows that (all the summations are over indices $(k_1, k_2, l_1, l_2)$),
	\begin{align*}
		\|\hat{\bm{a}}-\mathcal{E}(\hat{\bm{a}})\|_{\mathcal{R}}
		&\leq \sum \| a^{(k_1,k_2)}_{l_1 l_2} \big( \prod_d e^{\pi iv_d m_d}  
		 - \prod_d e^{\pi iv_d\mu_d} \ \mathcal{P}_{r_d-1, \xi(\epsilon,r_d), \pi}(v_d(m_d-\mu_d))  \big) \|_{\mathcal{R}} \\
		&= \sum |a^{(k_1,k_2)}_{l_1 l_2}| \cdot \| \prod_d e^{\pi iv_d(m_d-\mu_d)}-\prod_d \mathcal{P}_{r_d-1, \xi(\epsilon,r_d), \pi}(v_d(m_d-\mu_d)) \|_{\mathcal{R}}.
	\end{align*}
	Since $l_d$ ranges from $0$ to $q_d-1$, we have $|v_d| \leq 2 (q_d/2) /N_d=1/p_d$, 
	and therefore $M_d|v_d| \leq M_d/p_d \leq \xi(\epsilon,r_d)$.
	We extend the domain of the RHS to $(x_1,x_2) \in \prod_d [\mu_d-M_d,\mu_d+M_d] $
	and substitute $v_d(x_d-\mu_d)$ with $x_d'$:
	\begin{align*}
		\|\hat{\bm{a}}-\mathcal{E}(\hat{\bm{a}})\|_{\mathcal{R}}
		&\leq \sum |a^{(k_1,k_2)}_{l_1 l_2}| \cdot \| \prod_d e^{\pi ix_d'}- \prod_d \mathcal{P}_{r_d-1, \xi(\epsilon,r_d), \pi}(x_d')  
		\|_{ |x_d'| \leq M_d|v_d|,\ \forall d } \\
		&\leq \sum |a^{(k_1,k_2)}_{l_1 l_2}| \cdot \| \prod_d e^{\pi ix_d'}- \prod_d \mathcal{P}_{r_d-1, \xi(\epsilon,r_d), \pi}(x_d')  
		\|_{ |x_d'| \leq \xi(\epsilon,r_d),\ \forall d } .
	\end{align*}
	Note that, if $ |x_d'| \leq \xi(\epsilon,r_d)$ for $d=1,2$, then the following inequality holds:
	\begin{align*}
		&| \prod_d e^{\pi ix_d'} - \prod_d \mathcal{P}_{r_d-1, \xi(\epsilon,r_d), \pi}(x_d') | \\
		&=| (e^{\pi ix_1'} - \mathcal{P}_{r_1-1, \xi(\epsilon,r_1), \pi}(x_1')) \cdot e^{\pi ix_2'}
		+ \mathcal{P}_{r_1-1, \xi(\epsilon,r_1), \pi}(x_1') \cdot (e^{\pi ix_2'} - \mathcal{P}_{r_2-1, \xi(\epsilon,r_2), \pi}(x_2')) | \\
		&\leq | e^{\pi ix_1'} - \mathcal{P}_{r_1-1, \xi(\epsilon,r_1), \pi}(x_1') | \cdot |e^{\pi ix_2'}|
		+ | \mathcal{P}_{r_1-1, \xi(\epsilon,r_1), \pi}(x_1') |  \cdot |e^{\pi ix_2'} - \mathcal{P}_{r_2-1, \xi(\epsilon,r_2), \pi}(x_2')| \\
		&\leq \epsilon \cdot 1 + (\epsilon+1) \cdot \epsilon
	\end{align*}
	since $ |e^{\pi ix_2'}| = 1$ and $ | \mathcal{P}_{r_1-1, \xi(\epsilon,r_1), \pi}(x_1') | \leq | \mathcal{P}_{r_1-1, \xi(\epsilon,r_1), \pi}(x_1') - e^{\pi ix_1'} | + |e^{\pi ix_1'}| \leq \epsilon+1$.
	Therefore, we obtain the desired approximation bound of 2-dimensional \method:
	\[
		\|\hat{\bm{a}}-\mathcal{E}(\hat{\bm{a}})\|_{\mathcal{R}}
		\leq \sum |a^{(k_1,k_2)}_{l_1 l_2}| \cdot (\epsilon^2 + 2\epsilon) 
		=\| \bm{a} \| _{\rm1} \cdot (\epsilon^2 + 2\epsilon) .
	\]
\end{proof}